\newcommand{\inner}[2]{\left\langle #1, #2 \right\rangle}
\newcommand{\innerh}[2]{\inner{#1}{#2}_\mathcal{H}}

\documentclass[twoside, 11pt]{article}
\usepackage[hmargin=2cm,vmargin=3.5cm]{geometry}
\usepackage{algorithm2e,amsmath,amsthm,amssymb,graphicx}

\newtheorem{thm}{Theorem}[section]
\newtheorem{cor}[thm]{Corollary}
\newtheorem{lem}[thm]{Lemma}
\newtheorem{prop}[thm]{Proposition}

\theoremstyle{remark}
\newtheorem{remark}[thm]{Remark}
\theoremstyle{definition}
\newtheorem{defn}[thm]{Definition}

\let\hat\widehat
\let\tilde\widetilde

\begin{document}

\title{Differential Privacy for Functions and Functional Data}

\author{Rob Hall \\
        \texttt{rjhall@cs.cmu.edu} \\
        Machine Learning Department,\\
        Carnegie Mellon University,\\
        Pittsburgh, PA 15289, USA
        \and
        Alessandro Rinaldo \\
        \texttt{arinaldo@stat.cmu.edu} \\
        Department of Statistics,\\
        Carnegie Mellon University,\\
        Pittsburgh, PA 15289, USA
        \and
        Larry Wasserman \\
        \texttt{larry@stat.cmu.edu} \\
        Department of Statistics,\\
        Carnegie Mellon University,\\
        Pittsburgh, PA 15289, USA}

\maketitle

\begin{abstract}
Differential privacy is a framework
for privately releasing summaries of a database.
Previous work has focused mainly on methods for which
the output is a finite dimensional vector, or an element of some discrete set.
We develop methods for releasing functions
while preserving differential privacy.
Specifically, we show that adding an appropriate Gaussian process
to the function of interest yields
differential privacy.  When the functions lie in the same RKHS as the
Gaussian process, then the correct noise level is established by
measuring the ``sensitivity'' of the function in the RKHS norm.
As examples we consider kernel density estimation,
kernel support vector machines,
and functions in reproducing kernel Hilbert spaces.
\end{abstract}

\section{Introduction}

Suppose we have database $D$ which consists of measurements of a set of individuals.  We want to release a summary of $D$
without compromising the privacy of those individuals
in the database.
One framework for defining privacy rigorously in such problems
is {\em differential privacy}
\cite{DMNS06,Dwork:06}.
The basic idea is to produce an output via random noise addition.
An algorithm which does this may be thought of as inducing a distribution $P_D$
on the output space (where the randomness is due to
internal ``coin flips'' of the algorithm),
for every input data set $D$.
Differential privacy,
defined in Section \ref{section::ADE},
requires that $P_D$ not depend too strongly on any single
element of the database $D$.

The literature on differential privacy is vast.
Algorithms that preserve differential privacy
have been developed
for boosting, parameter estimation, clustering, logistic regression, SVM learning
and many other learning tasks.
See, for example,
\cite{boosting},
\cite{pplr},
\cite{smith:2008}, \cite{km.jmlr},
\cite{NRS07},
\cite{KLN+08},
\cite{barak2007paa},
and references therein.
In all these cases, the data (both the input and output) are assumed to be
real numbers or vectors.
In this paper we are concerned with a setting in which the output,
and possibly the input data set, consist of functions.

A concept that has been important in differential privacy is
the ``sensitivity'' of the output \cite{DMNS06}.  In the case
of vector valued output
the sensitivity is typically measured in the Euclidean norm or the
$\ell_1$-norm.  We find that when the output is a function the
sensitivity may be measured in terms of an RKHS norm.  To establish
privacy a Gaussian process may be added to the function with noise level
calibrated to the ``RKHS sensitivity'' of the output.

The motivation for considering function valued data is two-fold.
First, in some problems the data are naturally function valued,
that is, each data point is a function.
For example, growth curves, temperature profiles,
and economic indicators are often of this form.
This has given rise to a subfield of statistics
known as functional data analysis
(see, for instance \cite{ramsay}).
Second, even if the data are not functions,
we may want to release a data summary that is a function.
For example,
if the data $d_1,\ldots, d_n\in\mathbb{R}^d$ are a sample
from a distribution with density $f$
then we can estimate the density with
the kernel density estimator
$$
\hat f(x) = \frac{1}{n}\sum_{i=1}^n W\left(\frac{||x-d_i||}{h}\right), \quad x \in \mathbb{R}^d,
$$
where $W$ is a kernel  (see, for instance \cite{olive-non}) and $h > 0 $ is the bandwidth parameter.
The density estimator is useful for may tasks such as
clustering and classification.
We may then want to release a ``version''
of the density estimator $\hat f$ in a way which fulfills the criteria of differential privacy.
The utility of such a procedure goes beyond merely estimating the underlying density.
In fact, suppose the goal is to release a privatized database.
With a differentially private density estimator in hand, a large sample of data may be drawn from
that density.  The release of such a sample would inherit the differential privacy properties of the
density estimator: see, in particular,  \cite{WZ} and \cite{MKAGV08}.  This is a very attractive proposition,
since a differentially private sample of data could be used as the basis for
any number of statistical analyses which may have been brought to bear
against the original data (for instance, exploratory data analysis, model fitting, etc).

Histograms are an example of a density estimator that
has been ``privatized'' in previous literature
\cite{WZ,CDMT05}.
However, as density estimators, histograms are suboptimal
because they are not smooth.
Specifically, they do converge at the minimax rate
under the assumption that the true density is smooth.
The preferred method for density estimation in statistics
is kernel density estimation.
The methods developed in this paper lead to a private kernel density estimator.

In addition to kernel density estimation, there are a myriad of other scenarios in
which the result of a statistical analysis is a function.  For
example, the regression function or classification function from a supervised
learning task.  We demonstrate how the theory we develop may be applied in these
contexts as well.

{\em Outline.}
We introduce some notation and review the definition
of differential privacy
in Section \ref{section::ADE}.  We also give a demonstration
of a technique to achieve the differential privacy for a vector
valued output.
The theory for demonstrating differentially privacy of
functions is established in Section
\ref{section::functions}.
In Section
\ref{section::examples}
we apply the theory to the problems of kernel density estimation and
kernel SVM learning.  We also demonstrate how the theory may apply to a
broad class of functions (a Sobolev space).
Section
\ref{section::algorithms}
discusses possible algorithms for outputting  functions.
%We conclude in section~\ref{section::conclusion}.

\section{Differential Privacy}
\label{section::ADE}

Here we recall the definition of differential privacy and introduce some notation.
Let $D = (d_1,\ldots,d_n) \in \mathcal{D}$ be an input database in which
$d_i$ represents a row or an individual, and where $\mathcal{D}$ is the
space of all such databases of $n$ elements.  For two databases $D,D^\prime$,
we say they are ``adjacent'' or ``neighboring'' and write $D\sim D^\prime$
whenever both have the same number of elements, but differ in one element.
In other words, there exists a permutation of $D$
having Hamming distance of 2 to $D^\prime$.  In some other works databases
are called ``adjacent'' whenever one database contains the other together with
exactly one additional element.

%The idea behind privacy preserving data
%mining is to perform some operations based on a database
%$D\in\mathcal{D}$ in such a manner that the output does not reveal
%enough information so as to permit the identification of the
%individuals which comprise the database.

We may characterize a non-private algorithm in terms of the function it outputs,
e.g., $\theta:\mathcal{D}\to\mathbb{R}^d$.  Thus we write $\theta_D = \theta(D)$
to mean the output when the input database is $D$.  Thus, a computer program which
outputs a vector may be characterized as a family of vectors $\{\theta_D:D\in\mathcal{D}\}$, one for every possible input database.
 Likewise a computer program which is randomized may be characterized by the
distributions $\{P_D:D\in\mathcal{D}\}$ it induces on the output space (e.g.,
$\mathbb{R}^d$) when the input is $D$.
We consider randomized algorithms where the input is a database in
$\mathcal{D}$ and the output takes values in
a measurable space $\Omega$ endowed with the $\sigma$-field ${\cal A}$.
Thus, to each such algorithm there correspond the set of distributions
$\{P_D:\ D\in {\cal D}\}$ on $(\Omega,\mathcal{A})$ indexed by databases.
We phrase the definition of differential privacy using this characterization
of randomized algorithms.

\begin{defn}[Differential Privacy]
A set of distributions $\{P_D:D\in\mathcal{D}\}$ is called $(\alpha,\beta)$-differentially private, or said to ``achieve $(\alpha,\beta)$-DP'' whenever for all $D \sim D^\prime\in {\cal D}$ we have:

\begin{equation}
\label{eq_dp}
P_D(A) \leq e^\alpha P_{D^\prime}(A) + \beta,\ \
\forall A \in \mathcal{A},
\end{equation}

\noindent where $\alpha,\beta \geq 0$ are parameters, and $\mathcal{A}$ is the finest $\sigma$-field on which all $P_D$ are defined.
\end{defn}

Typically the above definition is called ``approximate differential privacy'' whenever $\beta > 0$, and ``$(\alpha,0)$-differential privacy'' is shortened to ``$\alpha$-differential privacy.''  It is important to note that the relation $D \sim D^\prime$ is symmetric, and so the inequality (\ref{eq_dp}) is required to hold when $D$ and $D^\prime$ are swapped.  Throughout this paper we take $\alpha \leq 1$, since this simplifies some proofs.

The $\sigma$-field $\mathcal{A}$
is rarely mentioned in the literature on differential privacy
but is actually quite important.
For example if we were to take
$\mathcal{A} = \{\Omega,\emptyset\}$ then the condition
(\ref{eq_dp}) is trivially satisfied by any randomized algorithm.
To make the definition as strong as possible
we insist that $\mathcal{A}$ be the finest
available $\sigma$-field on which the $P_D$ are defined.  Therefore when $\Omega$ is discrete the typical $\sigma$-field is $\mathcal{A} = 2^\Omega$
(the class of all subsets of $\Omega$),
and when $\Omega$ is a
space with a topology it is typical to use the completion of the
Borel $\sigma$-field
(the smallest $\sigma$-field containing all open sets).
We raise this point since when $\Omega$ is a space of functions, the choice
of $\sigma$-field is more delicate.

\subsection{Differential Privacy of Finite Dimensional Vectors}

\cite{Dwork:06b} give a technique to achieve approximate differential privacy
for general vector valued outputs in which the ``sensitivity'' may be bounded.
We review this below, since the result is important in the demonstration
of the privacy of our methods which output functions.
What follows in this section is a mild alteration to the technique
developed by \cite{Dwork:06b} and \cite{dp_recom}, in that the ``sensitivity''
of the class of vectors is measured in the Mahalanobis distance rather
than the usual Euclidean distance.

In demonstrating the differential privacy, we make use of the following lemma which is simply an explicit statement of an argument used in a proof by \cite{Dwork:06b}.

\begin{lem}
\label{lem_achieve_dp}
Suppose that, for all $D\sim D^\prime$, there exists a set $A^\star_{D,D^\prime} \in \mathcal{A}$ such that, for all $S \in \mathcal{A}$,
\begin{equation}
\label{eq_classic_dp}
S \subseteq A^\star_{D,D^\prime} \Rightarrow P_D(S) \leq e^\alpha P_{D^\prime}(S)
\end{equation}
and
\begin{equation}
\label{eq_astar_measure}
P_D(A^\star_{D,D^\prime}) \geq 1-\beta.
\end{equation}
Then the family $\{P_D\}$ achieves the $(\alpha,\beta)$-DP.
\end{lem}

\begin{proof}
\noindent Let $S \in \mathcal{A}$. Then,
\begin{eqnarray*}
P_D(S) & = & P_D(S\cap A^\star) + P_D(S\cap A^{\star C}) \leq  P_D(S\cap A^\star) + \beta \\
       & \leq & e^\alpha P_{D^\prime}(S\cap A^\star) + \beta  \leq e^\alpha P_{D^\prime}(S) + \beta.
\end{eqnarray*}
The first inequality is due to (\ref{eq_astar_measure}), the second is
due to (\ref{eq_classic_dp}) and the third is due to the subadditivity
of measures.
\end{proof}

\noindent The above result shows that, so long as there is a large enough (in terms of the measure $P_D$) set on which the $(\alpha,0)$-DP condition holds, then the approximate $(\alpha,\beta)$-DP is achieved.

\begin{remark}
If $(\Omega,\mathcal{A})$ has a $\sigma$-finite dominating measure $\lambda$, then for (\ref{eq_classic_dp}) to hold a sufficient condition is that the ratio of the densities be bounded on some set $A^\star_{D,D^\prime}$:

\begin{equation}
\label{eq_bounded_densities}
\forall a \in A^\star_{D,D^\prime}: \frac{dP_D}{d\lambda}(a) \leq e^\alpha \frac{dP_{D^\prime}}{d\lambda}(a).
\end{equation}
This follows from the inequality
$$
P_D(S) = \int_S \frac{dP_D}{d\lambda}(a)\ d\lambda(a)
\leq \int_S e^\alpha\frac{dP_{D^\prime}}{d\lambda}(a)\ d\lambda(a)
= e^\alpha P_{D^\prime}(S).
$$
\end{remark}

In our next result we show that approximate differential privacy is achieved via
(\ref{eq_bounded_densities}) when the output is a real vector, say $v_D = v(D)
\in \mathbb{R}^d$, whose dimension does not depend on the database
$D$.  An example is when the database elements $d_i \in \mathbb{R}^d$ and the output is
the mean vector $v(D) = n^{-1}\sum_{i=1}^nd_i$.

\begin{prop}
\label{prop_dp_vec}
Suppose that, for a positive definite symmetric matrix $M\in\mathbb{R}^{d\times d}$, the family of vectors $\{v_D: D\in\mathcal{D}\} \subset \mathbb{R}^d$ satisfies

\begin{equation}
\label{bound_cond}\sup_{D\sim D^\prime} \|M^{-1/2}(v_D-v_{D^\prime})\|_2 \leq \Delta.
\end{equation}

Then the randomized algorithm which, for input database $D$ outputs
$$
\tilde{v}_D = v_D + \frac{c(\beta)\Delta}{\alpha}Z,\quad Z \sim \mathcal{N}_d(0,M)
$$
achieves $(\alpha,\beta)$-DP whenever
\begin{equation}
\label{eq_c_def}
c(\beta) \geq  \sqrt{2\log{\frac{2}{\beta}}}.
\end{equation}
\end{prop}

\begin{proof}
Since the Gaussian measure on $\mathbb{R}^d$ admits the Lebesgue measure $\lambda$ as a $\sigma$-finite dominating measure
we consider the ratio of the densities
$$
\frac{dP_D(x)/d\lambda}{dP_{D^\prime}(x)/d\lambda} =
\text{exp}\left\{ \frac{\alpha^2}{2c(\beta)^2\Delta^2}
\left[ (x-v_{D^\prime})M^{-1}(x-v_{D^\prime}) - (x-v_D)^TM^{-1}(x-v_D)\right]\right\}.
$$
This ratio exceeds $e^\alpha$ only when
$$
2x^TM^{-1}(v_D-v_{D^\prime}) + v_{D^\prime}^TM^{-1}v_{D^\prime} - v_D^TM^{-1}v_D \geq 2\frac{c(\beta)^2\Delta^2}{\alpha}.
$$
We consider the probability of this set under $P_D$,
in which case we have $x = v_D + \frac{c(\beta)\Delta}{\alpha}M^{1/2}z$,
where $z$ is an isotropic normal with unit variance.  We have
$$
\frac{c(\beta)\Delta}{\alpha} z^TM^{-1/2}(v_D-v_{D^\prime}) \geq
\frac{c(\beta)^2\Delta^2}{\alpha^2} - \frac{1}{2}(v_D-v_{D^\prime})^TM^{-1}(v_D-v_{D^\prime}).
$$
Multiplying by $\frac{\alpha}{c(\beta)\Delta}$ and using (\ref{bound_cond}) gives
$$
z^TM^{-1/2}(v_D-v_{D^\prime}) \geq \frac{c(\beta)\Delta}{\alpha} - \frac{\alpha\Delta}{2c(\beta)}.
$$
Note that the left side is a normal random variable with mean zero and
variance smaller than $\Delta^2$.  The probability of this set is
increasing with the variance of said variable, and so we
examine the probability when the variance equals $\Delta^2$.  We also
restrict to $\alpha \leq 1$, and let $y \sim \mathcal{N}(0,1)$,
yielding
\begin{eqnarray*}
P\Biggl(z^TM^{-1/2}(v_D-v_{D^\prime}) \geq \frac{c(\beta)\Delta}{\alpha} - \frac{\alpha\Delta}{2c(\beta)}\Biggr)
&\leq & P\left(\Delta y \geq \frac{c(\beta)\Delta}{\alpha} - \frac{\alpha\Delta}{2c(\beta)}\right) \\
& \leq & P\left(y \geq c(\beta) - \frac{1}{2c(\beta)}\right)\\
&   \leq & \beta,\\
\end{eqnarray*}
where $c(\beta)$ is as defined in (\ref{eq_c_def}) and the final inequality is proved in \cite{Dwork:06}.
Thus lemma~\ref{lem_achieve_dp} gives the differential privacy.
\end{proof}

\begin{remark}
The quantity (\ref{bound_cond}) is a mild modification of the usual notion of
``sensitivity'' or ``global sensitivity'' \cite{DMNS06}.  It is nothing more than
the sensitivity measured in the Mahalanobis distance corresponding
to the matrix $M$.  The case $M=I$ corresponds to the usual
Euclidean distance, a setting that has been
studied previously by \cite{dp_recom}, among others.
\end{remark}

%We also observe that coordinate projections of released vectors will
%themselves satisfy the approximate differential privacy.  Let
%$\tilde{u}_D \in \mathbb{R}^p$ for $p < d$ have entries given by the
%first $p$ coordinates of $\Pi\tilde{v}_D$ for some permutation matrix
%$\Pi$.  Then we have for all $D \sim D^\prime$ and borel sets $B$ of
%$\mathbb{R}^p$:
%\begin{equation}
%\label{eq_dp_proj}
%P(\tilde{u}_D \in B) = P(\Pi\tilde{v}_D \in B\times\mathbb{R}^{d-p}) \leq
%e^\alpha P(\Pi\tilde{v}_{D^\prime} \in B\times\mathbb{R}^{d-p}) + \beta =
%e^\alpha P(\tilde{u}_{D^\prime} \in B) + \beta.
%\end{equation}

%\textbf{TODO} the above hints that for a suitable increasing sequence of vector valued problems, the sensitivity is increasing in the dimension.

\subsection{The Implications of Approximate Differential Privacy}

The above definitions provide a strong privacy guarantee in the sense
that they aim to protect against an adversary having almost complete
knowledge of the private database.  Specifically, an adversary knowing
all but one of the data elements and having observed the output of a
private procedure, will remain unable to determine the identity of the
data element which is unknown to him.  To see this, we provide an analog
of theorem~2.4 of \cite{WZ}, who consider the case of $\alpha$-differential privacy.

Let the adversary's database be denoted by $D_A=(d_1,\ldots,d_{n-1})$, and
the private database by $D = (d_1,\ldots,d_n)$.  First note that
before observing the output of the private algorithm, the
adversary could determine that the private database $D$ lay in the
set $\left\{ (d_1,\ldots,d_{n-1},d) \in \mathcal{D}\right\}.$ Thus,
the private database comprises his data with one more element.  Since
all other databases may be excluded from consideration by the
adversary we concentrate on those in the above set.  In particular, we obtain the following analog of
theorem 2.4 of \cite{WZ}.

\begin{prop}
Let $X\sim P_D$ where the family $P_D$
achieves the $(\alpha,\beta)$-approximate DP.
Any level $\gamma$ test of: $H: D = D_0$ vs $V: D \neq D_0$ has power bounded above by $\gamma e^\alpha + \beta$.
\end{prop}

\noindent
The above result follows immediately from noting that the rejection region of the
test is a measurable set in the space and so obeys the constraint
of the differential privacy.  The implication of the above proposition
is that the power of the test will be bounded close to its size.
When $\alpha,\beta$ are small, this means that the test is close to
being ``trivial'' in the sense that it is no more likely to correctly
reject a false hypothesis than it is to incorrectly reject the true
one.

\section{Approximate Differential Privacy for Functions}
\label{section::functions}

The goal of the release a function raises a number of questions.  First
what does it mean for a computer program to output a function?  Second, how can
the differential privacy be demonstrated?  In this section we continue to treat
randomized algorithms as measures, however now they are measures over function spaces.
In section~\ref{section::algorithms} we demonstrate concrete algorithms, which in
essence output the function on any arbitrary countable set of points.

We cannot expect the techniques for finite dimensional vectors
to apply directly when dealing with functions.  The reason is that $\sigma$-finite
dominating measures of the space of functions do not exist, and, therefore, neither do densities.
However, there exist probability measures on the spaces of
functions.  Below, we demonstrate the approximate differential privacy of
measures on function spaces, by considering random variables which
correspond to evaluating the random function on a finite set of
points.

We consider the family of functions over $T = \mathbb{R}^d$ (where appropriate we
may restrict to a compact subset such as the unit cube in $d$-dimensions):
$$
\{f_D: D \in\mathcal{D} \} \subset \mathbb{R}^T.
$$

A before, we consider randomized algorithms which on input $D$,
output some $\tilde{f}_D \sim P_D$
where $P_D$ is a measure on $\mathbb{R}^T$ corresponding to $D$.
The nature of the $\sigma$-field on this space will be
described below.

\subsection{Differential Privacy on the Field of Cylinders}

We define the ``cylinder sets'' of functions
(see \cite{Billingsley95}) for all finite subsets $S = (x_1,\ldots,x_n)$ of
$T$, and Borel sets $B$ of $\mathbb{R}^n$
$$
C_{S,B} = \left\{f \in \mathbb{R}^T:  (f(x_1),\ldots,f(x_n)) \in B\right\}.
$$
These are just those functions which take values in prescribed sets,
at those points in $S$.  The family of sets: $\mathcal{C}_S =
\{C_{S,B} : B\in\mathcal{B}(\mathbb{R}^n)\}$ forms a $\sigma$-field
for each fixed $S$, since it is the preimage of $\mathcal{B}(\mathbb{R}^n)$
under the operation of evaluation on
the fixed finite set $S$.  Taking the union over all finite
sets $S$ yields the collection
$$
\mathcal{F}_0 = \bigcup_{S : |S| < \infty}\mathcal{C}_S.
$$
This is a field (see \cite{Billingsley95} page
508) although not a $\sigma$-field, since it does not have the
requisite closure under countable intersections (namely it
does not contain cylinder sets for which $S$ is countably infinite).
We focus on the creation of algorithms for which the differential
privacy holds over the field of cylinder sets, in the sense that, for all $D \sim D^\prime \in \mathcal{D}$,
\begin{equation}
\label{eq_cylinder_dp}
 P(\tilde{f}_D \in A) \leq e^\alpha P(\tilde{f}_{D^\prime} \in A) + \beta, \quad \forall A \in \mathcal{F}_0.
\end{equation}

This statement appears to be prima facie unlike the definition (\ref{eq_dp}), since $\mathcal{F}_0$ is not
a $\sigma$-field on $\mathbb{R}^T$.
However, we
give a limiting argument which demonstrates that to satisfy
(\ref{eq_cylinder_dp}) is to achieve the approximate $(\alpha,\beta)$-DP throughout the
generated $\sigma$-field.  First we note that satisfying (\ref{eq_cylinder_dp}) implies that
the release of any finite evaluation of the function achieves the differential privacy.
Since for any finite $S\subset T$, we have
that $\mathcal{C}_S \subset \mathcal{F}_0$, we readily obtain the following result.

\begin{prop}
\label{prop_func_finite}
Let $x_1,\ldots,x_n$ be any finite set of points in $T$ chosen
a-priori.  Then whenever (\ref{eq_cylinder_dp}) holds, the release of
the vector
$$
\left(\tilde{f}_D(x_1),\ldots,\tilde{f}_D(x_n)\right)
$$
satisfies the $(\alpha,\beta)$-DP.
\end{prop}
\begin{proof}
We have that
$$P_D\left(\left(\tilde{f}(x_1),\ldots,\tilde{f}(x_n)\right) \in A\right) = P_D(\tilde{f} \in C_{\{x_1,\ldots,x_n\}, A}).$$
\noindent The claimed privacy guarantee follows from (\ref{eq_cylinder_dp}).
\end{proof}

We now give a limiting argument to extend (\ref{eq_cylinder_dp}) to the
generated $\sigma$-field (or, equivalently, the $\sigma$-field generated by the
cylinders of dimension 1)
$$
\mathcal{F} \stackrel{\text{def}}{=} \sigma(\mathcal{F}_0) = \bigcup_S \mathcal{C}_S
$$
where the union extends over all the countable subsets $S$ of $T$. The second equality above is due to \cite{Billingsley95} theorem 36.3
part ii.%, in which it is demonstrated that if $A\in \mathcal{F}$ then$A\in\mathcal{C}_S$ for some countable $S$.

Note that, for countable $S$, the cylinder sets take the form
$$
C_{S,B} = \left\{f \in \mathbb{R}^T : f(x_i) \in B_i,\ i = 1,2,\ldots \right\} = \bigcap_{i=1}^\infty C_{\{x_i\},B_i},
$$
where $B_i$'s are Borel sets of $\mathbb{R}$.

\begin{prop}
\label{prop_infinite}
Let (\ref{eq_cylinder_dp}) hold. Then, the family $\{P_D: D \in\mathcal{D}\}$ on $(\mathbb{R}^T,\mathcal{F})$ satisfies for all $D \sim D^\prime \in \mathcal{D}$:
\begin{equation}
\label{eq_field_dp}
 P_D(A) \leq e^\alpha P_{D^\prime}(A) + \beta, \quad \forall A \in \mathcal{F}.
\end{equation}
\end{prop}

\begin{proof}
\noindent Define
$C_{S,B,n} = \bigcap_{i=1}^n C_{\{t_i\},B_i}.$
Then, the sets $C_{S,B,n}$ form a sequence of sets which decreases towards $C_{S,B}$ and
$C_{S,B} = \lim_{n\to\infty}C_{S,B,n}.$
Since the sequence of sets is decreasing and the measure in question is a probability (hence bounded above by 1), we have
$$
P_D(C_{S,B}) = P_D(\lim_{n\to\infty}C_{S,B,n})
= \lim_{n\to\infty} P_D(C_{S,B,n}).
$$
Therefore, for each pair $D\sim D^\prime$ and for every $\epsilon > 0$,
there exists an $n_0$ so that for all $n \geq n_0$
$$
\left|P_D(C_{S,B})-P_D(C_{S,B,n})\right| \leq \epsilon,\quad
\left|P_{D^\prime}(C_{S,B})-P_{D^\prime}(C_{S,B,n})\right| \leq \epsilon.
$$
The number $n_0$ depends on whichever is the slowest sequence to converge.  Finally we obtain
\begin{align*}
P_D(C_{S,B}) &\leq P_D(C_{S,B,n_0}) + \epsilon \\
             &\leq e^\alpha P_{D^\prime}(C_{S,B,n_0}) + \beta + \epsilon\\
             &\leq e^\alpha P_{D^\prime}(C_{S,B}) + \beta + (1+e^\alpha)\epsilon\\
             &\leq e^\alpha P_{D^\prime}(C_{S,B}) + \beta + 3\epsilon.
\end{align*}
Since this holds for all $\epsilon > 0$ we conclude that
$P_D(C_{S,B}) \leq e^\alpha P_{D^\prime}(C_{S,B}) + \beta.$
\end{proof}

In principle, if it were possible for a computer to release a
complete description of the function $\tilde{f}_D$ then this result would
demonstrate the privacy guarantee achieved by our algorithm.  In practise
a computer algorithm which runs in a finite amount of time may only output
a finite set of points, hence this result is mainly of theoretical interest.
However, in the case in which the functions to be output are continuous, and the
restriction is made that $P_D$ are measures over $C[0,1]$ (the continuous
functions on the unit interval), another description of the $\sigma$-field becomes
available.  Namely, the restriction of $\mathcal{F}$ to the elements of $C[0,1]$
corresponds to the borel $\sigma$-field over $C[0,1]$ with the topology induced
by the uniform norm ($\|f\|_\infty = \sup_t|f(t)|$).  Therefore in the case of
continuous functions, differential privacy over $\mathcal{F}_0$ hence leads to
differential privacy throughout the borel $\sigma$-field.

In summary, we find that if every finite dimensional projection of the
released function satisfies  differential privacy, then so does
every countable-dimensional projection.  We now explore techniques which achieve the
differential privacy over these $\sigma$-fields.

\subsection{Differential Privacy via the Exponential Mechanism}

A straightforward means to output a function in a way which achieves the differential privacy is to make use of the so-called ``exponential mechanism'' of \cite{MT07}.  This approach entails the construction of a suitable finite set of functions $G = \{g_i,\ldots,g_m\} \in \mathbb{R}^T$, in which every $f_D$ has a reasonable approximation, under some distance function $d$.  Then, when the input is $D$, a function is chosen to output by sampling the set of $G$ with probabilities given by
$$
P_D(g_i) \propto \text{exp}\left\{\frac{-\alpha}{2}d(g_i,f_D) \right\}.
$$
\cite{MT07} demonstrate that such a technique achieves the $\alpha$-differential privacy, which is strictly stronger than the $(\alpha,\beta)$-differential privacy we consider here.  Although this technique is conceptually appealing for its simplicity, it remains challenging to use in practise since the set of functions $G$ may need to be very large in order to ensure the utility of the released function (in the sense of expected error).  Since the algorithm which outputs from $P_D$ must obtain the normalization constant to the distribution above, it must evidently compute the probabilities for each $g_i$, which may be extremely time consuming.  Note that techniques such as importance sampling are also difficult to bring to bear against this problem when it is important to maintain utility.

The technique given above can be interpreted as outputting a discrete random variable, and fulfilling privacy definition with respect to the $\sigma$-field consisting of the powerset of $G$.  This implies the privacy with respect to the cylinder sets, since the restriction of each cylinder set to the elements of $G$ corresponds some subset of $G$.

We note that the exponential mechanism above essentially corresponded to a discretization of the function space $\mathbb{R}^T$.  An alternative is to discretize the input space $T$, and to approximate the function by a piecewise constant function where the pieces correspond to the discretization of $T$.  Thereupon the approximation may be regarded as a real valued vector, with one entry for the value of each piece of the function.  This is conceptually appealing but it remains to be seen whether the sensitivity of such a vector valued output could be bounded.  In the next section we describe a method which may be regarded as similar to the above, and which has the nice property that the choice of discretization is immaterial to the method and to the determination of sensitivity.

\subsection{Differential Privacy via Gaussian Process Noise}

We propose to use measures $P_D$ over functions, which are Gaussian
processes.  The reason is that there is a strong connection between
these measures over the infinite dimensional function space, and the
Gaussian measures over finite dimensional vector spaces such as those
used in Proposition~\ref{prop_dp_vec}.  Therefore, with some additional
technical machinery which we will illustrate next, it is possible to move from differentially private
measures over vectors to those over functions.

A Gaussian process indexed by $T$  is a collection
of random variables $\{X_t : t \in T\}$, for which each finite subset
is distributed as a multivariate Gaussian (see, for instance, \cite{Adler90,Adler07}).  A sample from a Gaussian
process may be considered as a function $T\to\mathbb{R}$, by examining
the so-called ``sample path'' $t \to X_t$.  The Gaussian process is
determined by the mean and covariance functions, defined on $T$ and $T^2$ respectively, as
$$
m(t) = \mathbb{E}X_t,\quad K(s,t) = \text{Cov}(X_s,X_t).
$$
For any finite subset $S\subset T$,  the random vector  $\{X_t: t \in S\}$ has a normal distribution with the means,
variances, and covariances given by the above functions.  Such a
``finite dimensional distribution'' may be regarded as a projection of
the Gaussian process.  Below we propose particular mean and covariance
functions for which Proposition~\ref{prop_dp_vec} will hold for all
finite dimensional distributions.  These will require some
smoothness properties of the family of functions $\{f_D\}$.   We
 first demonstrate the technical machinery which allows us to move from
finite dimensional distributions to distributions on the function
space, and then we give differentially private measures on function spaces
of one dimension. Finally, we extend our results to multiple dimensions.

\begin{prop}
\label{prop_dp_func}
Let $G$ be a sample path of a Gaussian process having mean zero and
covariance function $K$.  Let $\{f_D: D\in\mathcal{D}\}$ be a family of functions
indexed by databases.  Then the release of
$$
\tilde{f}_D = f_D + \frac{\Delta c(\beta)}{\alpha}G
$$
is $(\alpha,\beta)$-differentially private (with respect to the cylinder $\sigma$-field $\mathcal{F}$) whenever

\begin{equation}
\label{finite_sens}
\sup_{D\sim D^\prime} \sup_{n < \infty} \sup_{(x_1,\ldots,x_n) \in T^n}
\left\|
\left(
\begin{array}{ccc}
K(x_1,x_1) & \cdots & K(x_1,x_n) \\
\vdots & \ddots & \vdots \\
K(x_n,x_1) & \cdots & K(x_n,x_n)
\end{array}
\right)^{-1/2}
\left(\begin{array}{c}
f_D(x_1) - f_{D^\prime}(x_1) \\ \vdots \\ f_D(x_n) - f_{D^\prime}(x_n)
\end{array}\right)\right\|_2
\leq \Delta.
\end{equation}
\end{prop}

\begin{proof}
For any finite set $(x_1,\ldots,x_n)\in T^n$, the vector
$\left(G(x_1),\ldots,G(x_n)\right)$ follows a multivariate normal
distribution having mean zero and covariance matrix specified by
$\text{Cov}(G(x_i),G(x_j)) = K(x_i,x_j)$.  Thus for the vector obtained
by evaluation of $\tilde{f}$ at those points, differential privacy is
demonstrated by Proposition~\ref{prop_dp_vec} since (\ref{finite_sens}) implies
the sensitivity bound (\ref{bound_cond}).
Thus, for any $n<\infty$ and any $(x_1,\ldots,x_n) \in T^n$ we have
$B\in\mathcal{B}(\mathbb{R}^n)$
$$
P_D\left(\left(\tilde{f}(x_1),\ldots,\tilde{f}(x_n)\right)\in B\right) \leq
e^\alpha P_{D^\prime}\left(\left(\tilde{f}(x_1),\ldots,\tilde{f}(x_n)\right)\in B\right) + \beta
$$

Finally note that for any $A\in\mathcal{F}_0$, we may write $A =
C_{X_n,B}$ for some finite $n$, some vector $X_n = (x_1,\ldots,x_n) \in T^n$ and some borel
set $B$.  Then
$$
P_D(\tilde{f}\in A) = P_D\left(\left(\tilde{f}(x_1),\ldots,\tilde{f}(x_n)\right)\in B\right).
$$
Combining this with the above gives the requisite privacy statement
 for all $A \in \mathcal{F}_0$.  Proposition~\ref{prop_infinite} carries
this to $\mathcal{F}$.
\end{proof}

\subsection{Functions in a Reproducing Kernel Hilbert Space}

When the family of functions lies in the reproducing kernel Hilbert
space (RKHS) which corresponds to the covariance kernel of the
Gaussian process, then establishing upper bounds of the form (\ref{finite_sens})
is simple.  Below, we give some basic definitions
for RKHSs, and refer the reader to \cite{rkhsprob} for a more detailed account.
We first recall that the
RKHS is generated from the closure of those functions which can be represented
as finite linear combinations of the kernel, i.e.,
$$
\mathcal{H}_0 = \left\{ \sum_{i = 1}^n \xi_i K_{x_i} \right\}
$$
for some finite $n$ and
sequence $\xi_i\in\mathbb{R},\ x_i \in T$, and where $K_x = K(x,\cdot)$.
For two functions $f = \sum_{i=1}^n\theta_iK_{x_i}$ and $g=\sum_{j=1}^m\xi_jK_{y_j}$ the inner product is given by
$$
\inner{f}{g}_\mathcal{H} = \sum_{i=1}^n\sum_{j=1^m}\theta_i\xi_jK(x_i,y_j),
$$
and the corresponding norm of $f$ is
$\|f\|_\mathcal{H} = \sqrt{\inner{f}{f}_\mathcal{H}}.$
This gives rise to the ``reproducing'' nature of the Hilbert space, namely,
$\inner{K_x}{K_y}_\mathcal{H} = K(x,y)$.
Furthermore, the functionals $\inner{K_x}{\cdot}_\mathcal{H}$ correspond to point evaluation, i.e.
$$
\inner{K_x}{f}_\mathcal{H} = \sum_{i=1}^n\theta_iK(x_i,x) = f(x).
$$
The RKHS $\mathcal{H}$ is then the closure of $\mathcal{H}_0$
with respect to the RKHS norm.  We now present the main theorem which suggests
an upper bound of the form required in Proposition~\ref{prop_dp_func}.

\begin{prop}
\label{prop_dp_rkhs}
For $f \in \mathcal{H}$, where $\mathcal{H}$ is the RKHS corresponding to the kernel $K$,
and for any finite sequence $x_1,\ldots,x_n$ of distinct points in $T$, we have:
$$
\left\|
\left(
\begin{array}{ccc}
K(x_1,x_1) & \cdots & K(x_1,x_n) \\
\vdots & \ddots & \vdots \\
K(x_n,x_1) & \cdots & K(x_n,x_n)
\end{array}
\right)^{-1/2}
\left(\begin{array}{c}
f(x_1) \\ \vdots \\ f(x_n)
\end{array}\right)\right\|_2
\leq \|f\|_\mathcal{H}.
$$
\end{prop}

The proof is in the appendix.
Together with Proposition~\ref{prop_dp_func}, this result implies the following.

\begin{cor}
\label{cor_dp_rkhs}
For $\{f_D: D \in\mathcal{D}\} \subseteq \mathcal{H}$, the release of
$$
\tilde{f}_D = f_D + \frac{\Delta c(\beta)}{\alpha}G
$$
is $(\alpha,\beta)$-differentially private
(with respect to the cylinder $\sigma$-field) whenever
$$
\Delta \geq \sup_{D\sim D^\prime}\|f_D-f_{D^\prime}\|_\mathcal{H}.
$$
and when $G$ is the sample path of a Gaussian process having mean zero and covariance
function $K$, given by the reproducing kernel of $\mathcal{H}$.
\end{cor}

\section{Examples}
\label{section::examples}

We now give some examples in
which the above technique may be used to construct private versions of functions in an RKHS.

\subsection{Kernel Density Estimation}
\label{sec::kde1}

Let $f_D$ be the kernel density estimator, where $D$ is
regarded as a sequence of points $x_i \in T$ as $i=1,\ldots,n$
drawn from a distribution with density $f$.
Let $h$ denote the bandwidth.
Assuming a Gaussian kernel, the estimator is
$$
f_D(x) = \frac{1}{n(2\pi h^2)^{d/2}}\sum_{i=1}^n\text{exp}\left\{\frac{-\|x-x_i\|_2^2}{2h^2} \right\}, \quad x \in T.
$$

Let
$D\sim D^\prime$ so that $D^\prime = x_1,\ldots,x_{n-1}, x_n^\prime$
(no loss of generality is incurred by demanding that the data
sequences differ in their last element). Then,
$$
(f_D - f_{D^\prime})(x) =
\frac{1}{n(2\pi h^2)^{d/2}}
\left( \text{exp}\left\{-\frac{\|x-x_n\|_2^2}{2h^2}\right\} -
\text{exp}\left\{-\frac{\|x-x_n^\prime\|_2^2}{2h^2}\right\} \right).
$$
If we use the Gaussian kernel as the covariance function for the
Gaussian process then upper bounding the RKHS norm of this function is trivial.
Thus, let
$K(x,y) = \text{exp}\left\{ -\frac{\|x-y\|_2^2}{2h^2} \right\}$.
Then
$f_D - f_{D^\prime} = \frac{1}{n(2\pi h^2)^{d/2}}\left( K_{x_n} - K_{x_n^\prime} \right)$
and
\begin{align*}
\|f_D - f_{D^\prime}\|^2_\mathcal{H} &=
\left(\frac{1}{n(2\pi h^2)^{d/2}}\right)^2\left(K(x_n,x_n) + K(x_n^\prime, x_n^\prime) -2K(x_n,x_n^\prime)\right) \\
&\leq 2\left(\frac{1}{n(2\pi h^2)^{d/2}}\right)^2.
\end{align*}
If we release
$$
\tilde{f}_D = f_D + \frac{c(\beta)\sqrt{2}}{\alpha n (2\pi h^2)^{d/2}}G
$$
where $G$ is a sample path of a Gaussian process having mean zero and covariance $K$,
then differential privacy is demonstrated by corollary~\ref{cor_dp_rkhs}.
We may compare the utility of the released estimator
to that of the non-private version.  Under standard smoothness assumptions on $f$, it is well-known (see \cite{olive-non}) that
the risk is
$$
R=\mathbb{E}\int (f_D(x) - f(x))^2 dx =
c_1 h^4 + \frac{c_2}{n h^d},
$$
for some constants $c_1$ and $c_2$.
The optimal bandwidth is
$h \asymp (1/n)^{1/(4+d)}$
in which case
$R = O( n^{-\frac{4}{4+d}})$.

For the differentially private function it is easy to see that
$$
\mathbb{E}\int (\tilde{f}_D(x) - f(x))^2 dx =
O\left( h^4 + \frac{c_2}{n h^d}\right).
$$
Therefore, at least in terms of rates,
no accuracy has been lost.

\begin{figure}[h]
\centering
\includegraphics[width=0.8\textwidth]{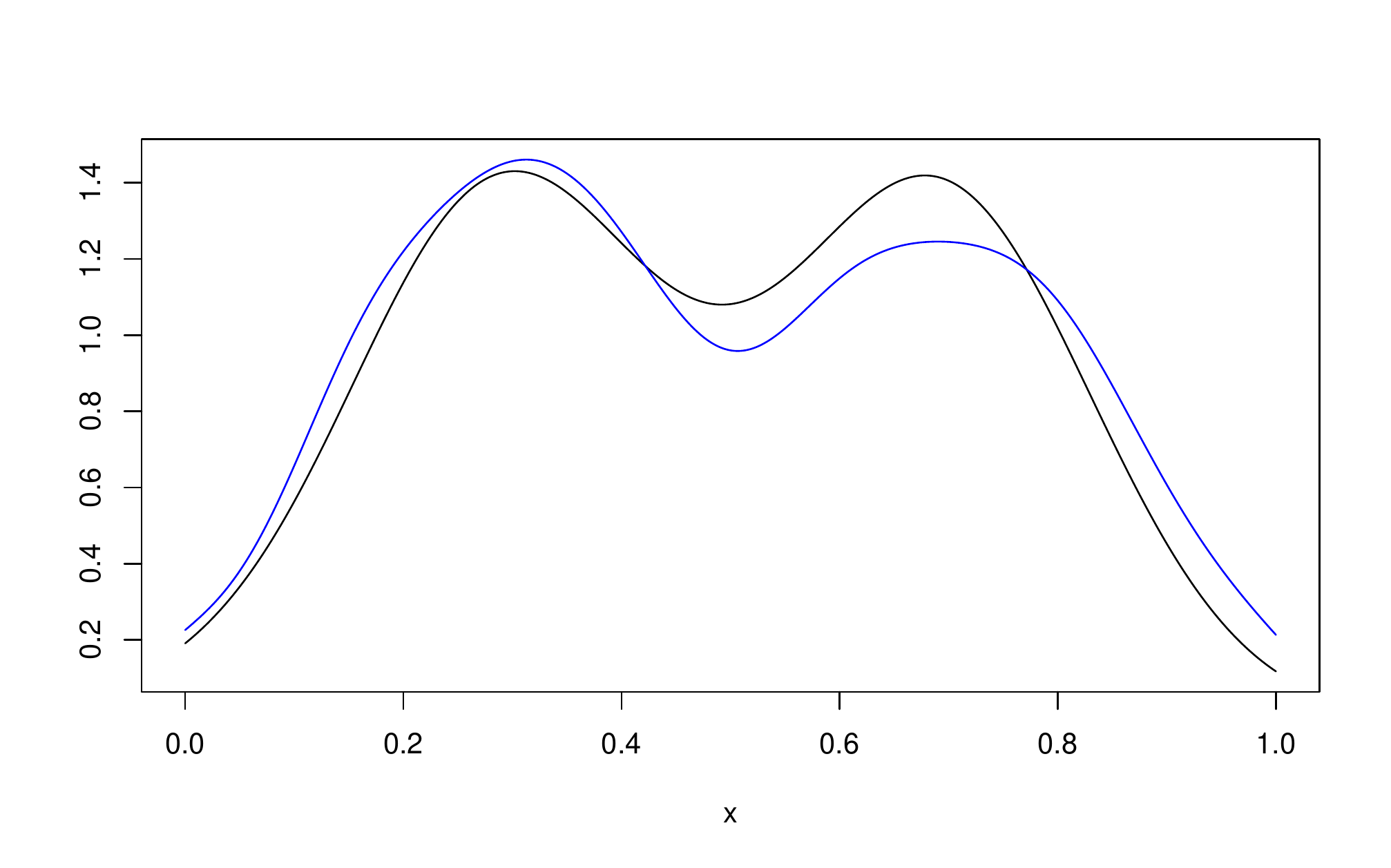}
\caption{An example of a kernel density estimator (the black curve)
  and the released version (the blue curve).  This
  uses the method developed in Section~\ref{sec::kde1}.  Here we sampled $n=100$
  points from a mixture of two normals centered at 0.3 and 0.7
  respectively.  We use $h=0.1$ and have $\alpha=1$ and $\beta =
  0.1$. The Gaussian Process is evaluated on an evenly spaced grid of
  1000 points between 0 and 1.  Note that gross features of the
  original kernel density estimator remain, namely the two peaks.}
  \label{fig_kde_rkhs}
\end{figure}

\subsubsection{Non-Isotropic Kernels}

The above demonstration of privacy also holds when the kernel is replaced
by a non-isotropic Gaussian kernel. In this case the kernel density estimate may take the form

$$
f_D(x) = \frac{1}{n(2\pi)^{d/2}|H|^{1/2}}\sum_{i=1}^n\text{exp}\left\{-\frac{1}{2}(x-x_i)^TH^{-1}(x-x_i) \right\}, \quad x \in T,
$$

\noindent where $H$ is a positive definite matrix and $|H|$ is the
determinant. For example it may
be required to employ a different choice of bandwidth for each
coordinate of the space, in which case $H$ would be a diagonal matrix
having non-equal entries on the diagonal.  So long as $H$ is fixed a-priori, privacy may be established by adding a Gaussian process having mean zero and covariance given by

$$
K(x,y) = \text{exp}\left\{-\frac{1}{2}(x-y)^TH^{-1}(x-y) \right\}.
$$

\noindent As above, the sensitivity is upper bounded, as

$$
\|f_D - f_{D^\prime}\|^2_\mathcal{H} \leq 2\left(\frac{1}{n(2\pi)^{d/2}|H|^{1/2}}\right)^2.
$$

Therefore it satisfies the $(\alpha,\beta)$-DP to release
$$
\tilde{f}_D = f_D + \frac{c(\beta)\sqrt{2}}{\alpha n(2\pi)^{d/2}|H|^{1/2}}G,
$$
where $G$ is a sample path of a Gaussian process having mean zero and covariance $K$.

\subsubsection{Private Choice of Bandwidth}

Note that the above assumed that $h$ (or $H$) was fixed a-priori by
the user.  In
usual statistical settings $h$ is a parameter that is tuned depending on
the data (not simply set to the correct order of growth as a function of
$n$).  Thus rather than fixed $h$ the user would use $\hat{h}$ which
depends on the data itself.  In order to do this it is necessary to find
a differentially private version of $\hat{h}$ and then to employ the
composition property of differential privacy (citation).

The typical way that the bandwidth is selected is by employing
the leave-one-out cross validation.  This consists of
choosing a grid of candidate values for $h$, evaluating the
leave one out log likelihood for each value, and then choosing whichever
is the maximizer.  This technique may be amenable to private
analysis via the ``exponential mechanism'' of (citation), however it would
evidently require that $T$ be a compact set which is known a-priori.
An alternative is to use a ``rule of thumb''
 (see \cite{scott_density}) for determining
the bandwidth which is given by
$$
\hat{h}_j = \left(\frac{4}{(d+1)n}\right)^{\frac{1}{d+4}} \frac{IQR_j}{1.34}
$$

\noindent In which $IQR_j$ is the observed interquartile range of the data
along the $j^{th}$ coordinate.  Thus this method gives a diagonal matrix $H$ as in the above section.  To make a private version $\tilde{h}_j$ we may use the technique of \cite{DL09} in which a differentially private algorithm for the interquartile range was developed.

\subsection{Functions in a Sobolev Space}
\label{sec::sobolev}

The above technique worked easily since we chose a
particular RKHS in which we knew the kernel density estimator to live.
What's more, since the functions themselves lay in the generating set
of functions for that space, the determination of the norm of the
difference $f_D-f_{D^\prime}$ was extremely simple.  In general we may
not be so lucky that the family of functions is amenable to such
analysis.  In this section we demonstrate a more broadly applicable
technique which may be used whenever the functions are sufficiently
smooth.  Consider the Sobolev space
$$
H^1[0,1] = \left\{ f\in C[0,1] : \int_0^1 (\partial f(x))^2 \ d\lambda(x) < \infty \right\}.
$$
This is a RKHS with the kernel
$K(x,y) = \text{exp}\left\{-\gamma\ |x-y|\right\}$
for positive constant $\gamma$.  The norm in this space is given by
\begin{equation}
\label{eq_norm_sobolev}
\|f\|^2_\mathcal{H} = \frac{1}{2}\left(f(0)^2 + f(1))^2\right) + \frac{1}{2\gamma}\int_0^1(\partial f(x))^2 +
\gamma^2f(t)^2\ d\lambda(t).
\end{equation}
See e.g., \cite{rkhsprob} (p. 316) and \cite{parzen61} for details.
Thus for a family of functions in
one dimension which lay in the Sobolev space $H^1$, we may determine a
noise level necessary to achieve the differential privacy by bounding
the above quantity for the difference of two functions.  For functions
over higher dimensional domains (as $[0,1]^d$
for some $d > 1$) we may construct an RKHS by taking the $d$-fold tensor
product of the above RKHS (see, in particular \cite{parzen63, aronszajn50} for details on the construction).  The resulting space has the reproducing kernel

$$
K(x,y) = \text{exp}\left\{ -\gamma\|x-y\|_1\right\},
$$

\noindent and is the completion of the set of functions

$$\mathcal{G}_0 = \left\{ f : [0,1]^d \to \mathbb{R} : f(x_1,\ldots,x_d) = f_1(x_1)\cdots f_d(x_d), f_i \in H^1[0,1] \right\}.$$

\noindent The norm over this set of functions is given by:

\begin{equation}
\label{eq_norm_tensor}
\|f\|^2_{\mathcal{G}_0} = \prod_{j=1}^d \|f_i\|^2_\mathcal{H}.
\end{equation}

The norm over the completed space agrees with the above on $\mathcal{G}_0$.
The explicit form is obtained by substituting (\ref{eq_norm_sobolev}) into
the right hand side of (\ref{eq_norm_tensor}) and replacing all instances of $\prod_{j=1}^df_j(x_j)$ with $f(x_1,\ldots,x_j)$.  Thus the norm in the completed space is defined for all $f$ possessing all first partial derivatives which are all in $\mathcal{L}_2$.

We revisit
the example of a kernel density estimator (with an isotropic
Gaussian kernel).  We note that this isotropic kernel function
is in the set $G_0$ defined above, as

$$
\phi_{\mu,h}(x) = \frac{1}{(2\pi h^2)^{d/2}}
\text{exp}\left\{-\frac{\|x-\mu\|_2^2}{2h^2}\right\} =
\prod_{j=1}^d \frac{1}{\sqrt{2\pi h}}
\text{exp}\left\{-\frac{(x_j-\mu_j)^2}{2h^2}\right\} = \prod_{j=1}^d\phi_{\mu_j,h}(x_j).
$$

\noindent Where $\phi_{\mu,h}$ is the isotropic Gaussian kernel on $\mathbb{R}^d$ with mean vector $\mu$ and $\phi_{\mu_j,h}$ is the Gaussian kernel in one dimension with mean $\mu_j$.  We obtain the norm of the latter one dimensional function by bounding the elements of the sum in (\ref{eq_norm_sobolev}) ad follows:
\begin{eqnarray*}
\int_0^1 (\partial \phi_{\mu_j,h}(x))^2\ d\lambda(x) & \leq & \int_{-\infty}^\infty \left(\partial \frac{1}{\sqrt{2\pi} h}e^{-(x-\mu_j)^2/2h^2} \right)^2\ d\lambda(x) =
\frac{1}{4\sqrt{\pi}h^3},
\end{eqnarray*}

\noindent and

\begin{eqnarray*}
\int_0^1 \phi_{\mu_j,h}(x)^2\ d\lambda(x) & \leq &  \int_{-\infty}^\infty \frac{1}{2\pi h^2}e^{-(x-\mu_j)^2/h^2} \ d\lambda(x) = \frac{1}{2\sqrt{2\pi}h},
\end{eqnarray*}
where we have used the fact that
$$
\phi_{\mu_j,h}(x)^2 \leq \frac{1}{2\pi h}, \quad \forall x \in \mathbb{R}^d.
$$
Therefore, choosing $\gamma=1/h$ leads to

$$
\|\phi_{\mu_j,h}\|_\mathcal{H}^2 \leq
\frac{1}{2\pi h^2} + \frac{1}{8\sqrt{\pi}h^2} + \frac{1}{4\sqrt{2\pi}h^2}
\leq \frac{1}{\sqrt{2\pi}h^2},
$$

and

$$
\|\phi_{\mu,h}\|_\mathcal{H}^2 \leq \frac{1}{(2\pi)^{d/2}h^{2d}}.
$$

\noindent Finally,

$$
\|f_D-f_{D^\prime}\|_\mathcal{H} = n^{-1}\|\phi_{x_n,h} - \phi_{x_n^\prime,h}\|_\mathcal{H} \leq \frac{2}{(2\pi)^{d/4}nh^d}
$$

Therefore, we observe a technique which attains higher
generality than the ad-hoc analysis of the preceding section.  However
this is at the expense of the noise level, which grows at a higher rate as $d$ increases.  An example of the technique
applied to the same kernel density estimation problem as above is given in Figure~\ref{fig_kde_sobolev}.

\begin{figure}[h]
\centering
\includegraphics[width=0.8\textwidth]{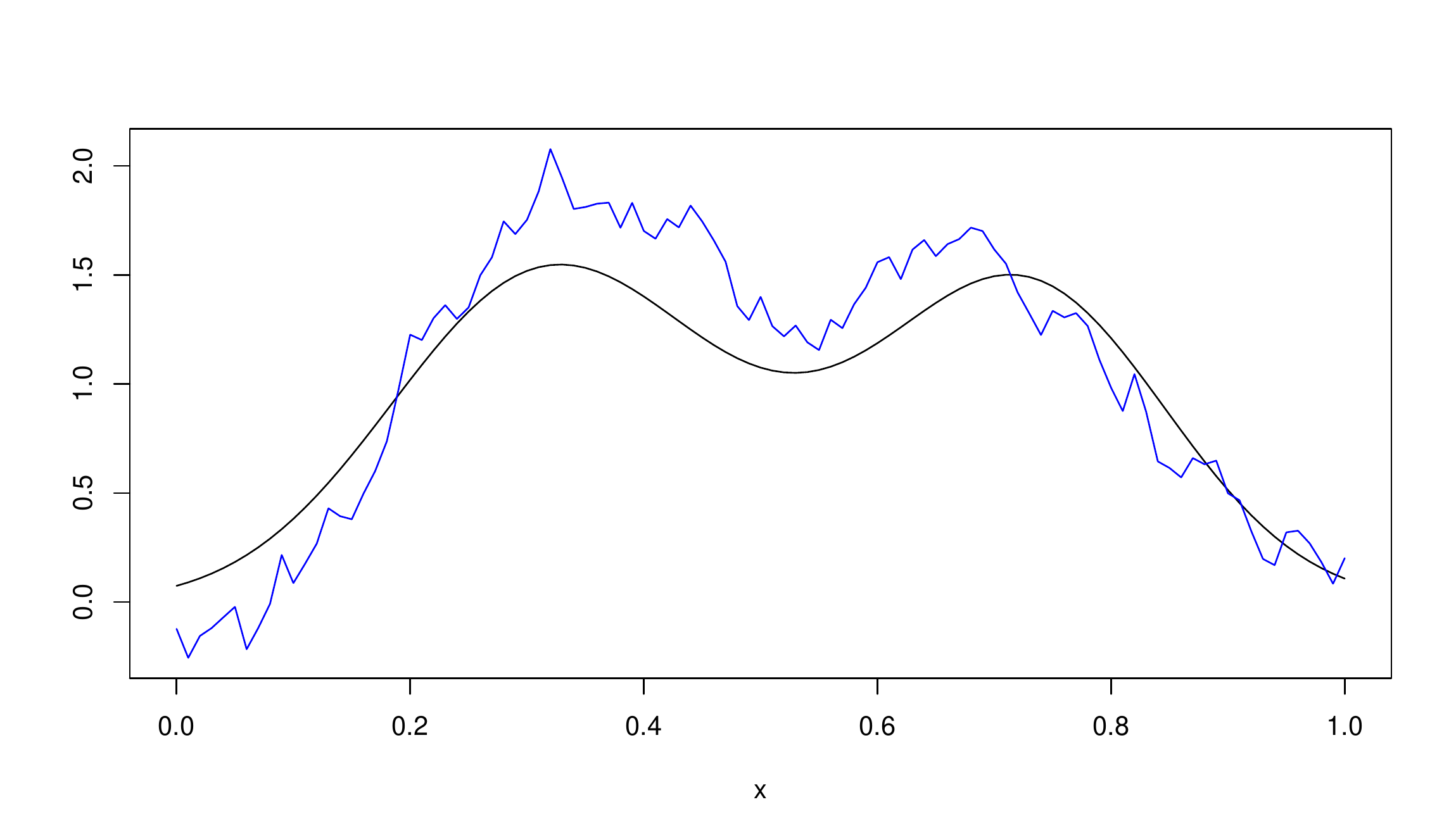}
\caption{An example of a kernel density estimator (the black curve)
  and the released version (the blue curve).  The setup is the same as
  in Figure~\ref{fig_kde_rkhs},  but the privacy mechanism developed in Section~\ref{sec::sobolev} was used instead.
  Note that the released function does not have the
  desirable smoothness of released function from Figure~\ref{fig_kde_rkhs}.}
  \label{fig_kde_sobolev}
\end{figure}

\subsection{Minimizers of Regularized Functionals in an RKHS}

The construction of the following section is due to \cite{bousquet02},
who were interested in determining the sensitivity of certain kernel
machines (among other algorithms) with  the aim of bounding the
generalization error of the output classifiers.  \cite{rubinstein}
noted that these
bounds are useful for establishing the noise level required for differential
privacy of support vector machines.
They are also useful for our approach to privacy in a function space.

We consider classification and regression schemes in which the
datasets $D = \{z_1,\ldots,z_n\}$ with $z_i = (x_i,y_i)$, where $x_i
\in [0,1]^d$ are some covariates, and $y_i$ is some kind of label,
either taking values on $\{-1,+1\}$ in the case of classification or
some taking values in some interval when the goal is regression.  Thus
the output functions are from $[0,1]^d$ to a subset of $\mathbb{R}$.
The functions we are interested in take the form
\begin{equation}\label{eq_opt}
f_D = \arg\min_{g \in \mathcal{H}} \frac{1}{n} \sum_{z_i\in D}\ell(g, z_i) + \lambda \|g\|_\mathcal{H}^2
\end{equation}
where $\mathcal{H}$ is some RKHS to be determined, and $\ell$ is the
so-called ``loss function.''  We now recall a definition from
\cite{bousquet02} (using $M$ in place of their $\sigma$ to prevent confusion):

\begin{defn}[$M$-admissible loss function: see \cite{bousquet02}]
A loss function: $\ell(g,z) = c(g(x),y)$ is called $M$-admissible whenever $c$ it is
convex in its first argument and
Lipschitz with constant $M$ in its first argument.
\end{defn}

We will now demonstrate that for (\ref{eq_opt}), whenever the loss
function is admissible, the minimizers on adjacent datasets may be
bounded close together in RKHS norm.  Denote the part of the
optimization due to the loss function:
$$
L_D(f) = \frac{1}{n} \sum_{z_i \in D}\ell(f, z_i).
$$
Using the technique from the proof of lemma 20 of \cite{bousquet02} we
find that since $\ell$ is convex in its first argument we have
$$
L_D(f_D + \eta \delta_{D^\prime,D}) - L_D(f_D) \leq \eta(L_D(f_{D^\prime}) - L_D(f_D)),
$$
\noindent where $\eta\in[0,1]$ and we use $\delta_{D^\prime,D} = f_{D^\prime} - f_{D}$.
This also holds when $f_D$ and $f_{D^\prime}$ swap places.  Summing
the resulting inequality with the above and rearranging yields
$$
L_D(f_{D^\prime} - \eta \delta_{D^\prime,D}) -  L_D(f_{D^\prime}) \leq L_D(f_D) - L_D(f_D + \eta \delta_{D^\prime,D}).
$$
Due to the definition of $f_D,f_{D^\prime}$ as the minimizers of their respective functionals we have
\begin{align*}
L_D(f_D) + \lambda\|f_D\|_\mathcal{H}^2 &\leq
L_D(f_D + \eta\delta_{D^\prime,D}) + \lambda\|f_D + \eta\delta_{D^\prime,D}\|_\mathcal{H}^2 \\
L_{D^\prime}(f_{D^\prime}) + \lambda\|f_{D^\prime}\|_\mathcal{H}^2 &
\leq L_{D^\prime}(f_{D^\prime} - \eta \delta_{D^\prime,D}) + \lambda\|f_{D^\prime} - \eta \delta_{D^\prime,D}\|_\mathcal{H}^2.
\end{align*}
This leads to the inequalities
\begin{align*}
0 & \geq  \lambda\left(\|f_D\|_\mathcal{H}^2 - \|f_D + \eta\delta_{D^\prime,D}\|_\mathcal{H}^2  + \|f_{D^\prime}\|_\mathcal{H}^2 - \|f_{D^\prime} - \eta \delta_{D^\prime,D}\|_\mathcal{H}^2   \right) \\
 & \quad + L_D(f_D) - L_D(f_D + \eta\delta_{D^\prime,D} +  L_{D^\prime}(f_{D^\prime})- L_{D^\prime}(f_{D^\prime} - \eta \delta_{D^\prime,D}) \\
& \geq 2\lambda\|\eta\delta_{D^\prime,D}\|_\mathcal{H}^2 - L_D(f_{D^\prime}) + L_D(f_{D^\prime} - \eta\delta_{D^\prime,D}) +  L_{D^\prime}(f_{D^\prime})- L_{D^\prime}(f_{D^\prime} - \eta \delta_{D^\prime,D}) \\
& = 2\lambda\|\eta\delta_{D^\prime,D}\|_\mathcal{H}^2 + \frac{1}{n}\left(\ell(z,f_{D^\prime}) - \ell(z,f_{D^\prime} - \eta \delta_{D^\prime,D}) + \ell(z^\prime,f_{D^\prime}) - \ell(z^\prime,f_{D^\prime} - \eta \delta_{D^\prime,D})\right).
\end{align*}

\noindent Moving the loss function term to the other side and using the Lipschitz property we finally obtain that

$$\|f_D-f_{D^\prime}\|_\mathcal{H}^2 \leq \frac{M}{\lambda n}\|f_D - f_{D^\prime}\|_\infty.$$

\noindent What's more, the reproducing property together with Cauchy-Schwarz inequality yields
$$
|f_D(x) - f_{D^\prime}(x)| = |\innerh{f_D-f_{D^\prime}}{K_x}| \leq
\|f_D-f_{D^\prime}\|_\mathcal{H}\sqrt{K(x,x)}.
$$
Combining with the previous result gives
$$
\|f_D-f_{D^\prime}\|_\mathcal{H}^2 \leq \frac{M}{\lambda n}\|f_D-f_{D^\prime}\|_\mathcal{H}\sqrt{\sup_xK(x,x)},
$$
which, in turn, leads to
$$
\|f_D-f_{D^\prime}\|_\mathcal{H} \leq \frac{M}{\lambda n}\sqrt{\sup_xK(x,x)}.
$$
For a soft-margin kernel SVM we have the loss function: $\ell(g,z) =
(1-yg(x))_+$, which means the positive part of the term in
parentheses.  Since the label $y$ takes on either plus or minus one,
we find this to be $1$-admissible.  An example of a kernel SVM in $T=\mathbb{R}^2$ is
shown in Figure~\ref{fig_svm_demo}.

\begin{figure}[h!]
  \centering
      \includegraphics[width=0.8\textwidth]{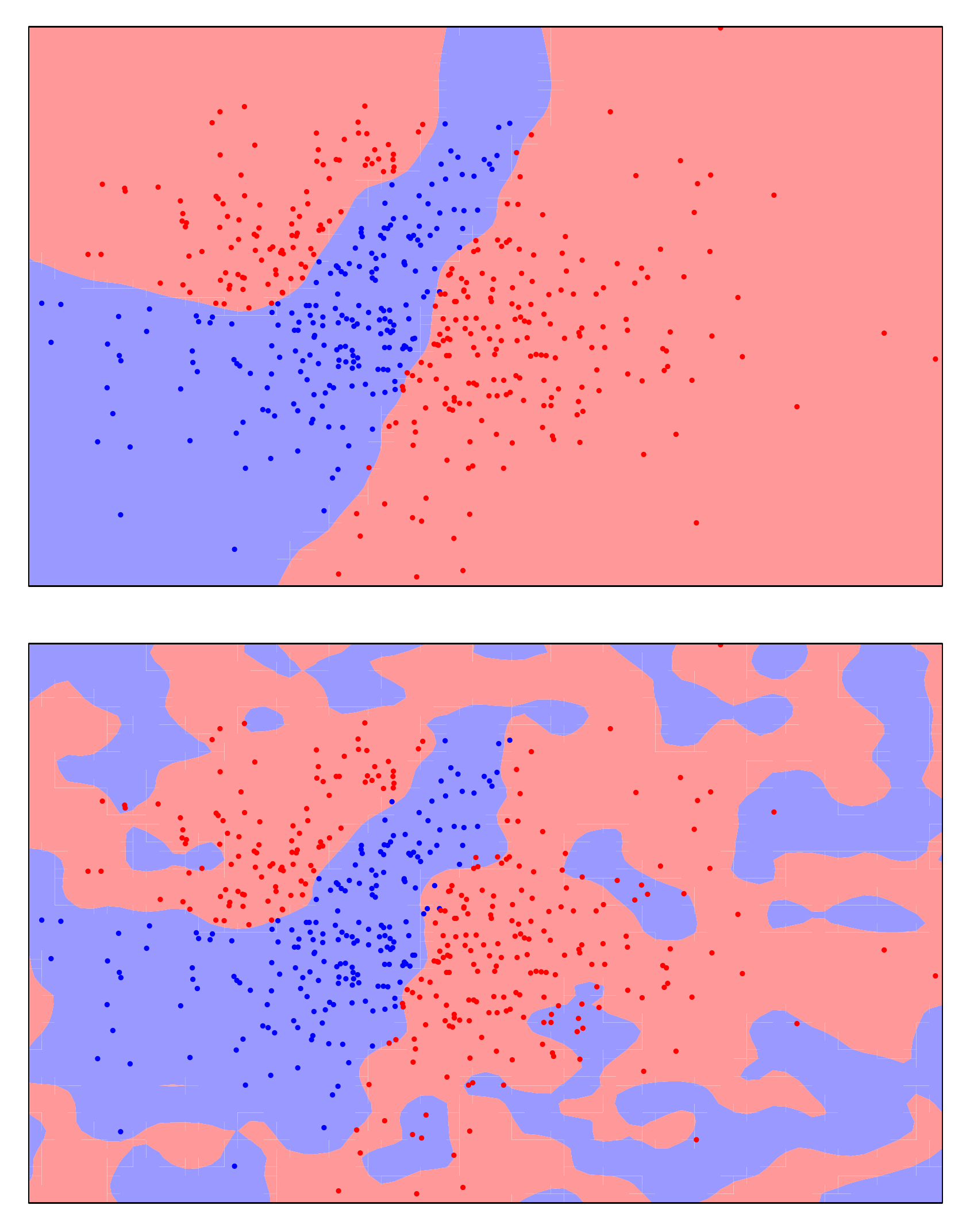}
  \caption{An example of a kernel support vector machine.  In the top image are the data points, with the colors representing the two class labels.  The background color corresponds to the class predicted by the learned kernel svm.  In the bottom image are the same data points, with the predictions of the private kernel svm.  This example uses the Gaussian kernel for classification. }
  \label{fig_svm_demo}
\end{figure}

\section{Algorithms}
\label{section::algorithms}

There are two main modes in which functions $f_D$ could
be released by the holder of the data $D$ to the outside parties.  The
first is a ``batch'' setting in which the parties
designate some finite collection of points $x_1\ldots,x_n \in T$.  The
database owner computes $\tilde{f}_D(x_i)$ for each $i$ and
return the vector of results.  At this point the entire transaction
would end with only the collection of pairs $(x_i, \tilde{f}_D(x_i))$
being known to the outsiders.  An alternative is the ``online''
setting in which outside users repeatedly specify points in $x_i \in
T$, the database owner replies with $\tilde{f}_D(x_i)$, but unlike
the former setting he remains available to respond to more
requests for function evaluations.  We name these settings ``batch''
and ``online'' for their resemblance of the batch and online settings
typically considered in machine learning algorithms.

The batch method is nothing more than sampling a
multivariate Gaussian, since the set $x_1,\ldots,x_n \in T$ specifies
the finite dimensional distribution of the Gaussian process from which
to sample.  The released vector is simply
$$
\left(\begin{array}{c} \tilde{f}_D(x_1) \\ \vdots \\ \tilde{f}_D(x_n) \end{array}\right) \sim \mathcal{N}\left(
\left(\begin{array}{c} f_D(x_1) \\ \vdots \\ f_D(x_n) \end{array}\right),\
\frac{c(\beta)\Delta}{\alpha}\left(\begin{array}{ccc}
K(x_1,x_1) & \cdots & K(x_1,x_n) \\
\vdots & \ddots & \vdots \\
K(x_n,x_1) & \cdots & K(x_n,x_n)
\end{array}
\right)
\right).
$$
In the online setting, the data owner upon receiving a
request for evaluation at $x_i$ would sample the gaussian process
conditioned on the samples already produced at $x_1,\ldots,x_{i-1}$.
Let
$$
C_i = \left(\begin{array}{ccc}
K(x_1,x_1) & \cdots & K(x_1,x_{i-1}) \\
\vdots & \ddots & \vdots \\
K(x_{i-1},x_1) & \cdots & K(x_{i-1},x_{i-1})
\end{array}
\right),\quad
G_i = \left(\begin{array}{c} \tilde{f}_D(x_1) \\ \vdots \\ \tilde{f}_D(x_{i-1}) \end{array}\right), \quad
V_i = \left(\begin{array}{c} K(x_1,x_i) \\ \vdots \\ K(x_{i-1}, x_i) \end{array}\right).
$$
Then,
$$
\tilde{f}_D(x_i) \sim \mathcal{N} \left(
V_i^TC_i^{-1}G_i,\ K(x_i,x_i) - V_i^TC_i^{-1}V_i
\right).
$$
The database owner may track the inverse matrix $C_i^{-1}$
and after each request update it into $C_{i+1}^{-1}$ by making use of
Schurs Complements combined with the matrix inversion lemma.
Nevertheless we note that as $i$ increases the computational
complexity of answering the request will in general grow.  In the very
least, the construction of $V_i$ takes time proportional to $i$.  This
may make this approach problematic to implement in practise.  However
we note that when using the covariance kernel
$$
K(x,y) = \text{exp}\left\{ -\gamma \ |x-y|_1 \right\}
$$
that a more efficient algorithm presents itself.  This is the kernel
considered in section~\ref{sec::sobolev}.  Due to the above
form of $K$, we find that for $x < y < z$ we have: $K(x,z) =
K(x,y)K(y,z)$.  Therefore in using the above algorithm we would find
that $V_i$ is always contained in the span of at most two rows of
$C_i$.  This is most evident when, for instance, $x_i < \min_{j < i}x_j$.  In
this case let $m = \arg\min_{j < i}x_j$ $V_i = K(x_i,x_m) C_i(m)$, in
which $C_i(m)$ means the $m^{th}$ row of $C_i$.  Therefore
$C_i^{-1}V_i$ will be a sparse vector with exactly one non-zero entry
(taking value $K(x,x_m)$) in the $m^{th}$ position.  Similar algebra
applies whenever $x_i$ falls between two previous points, in which
case $V_i$ lays in the span of the two rows corresponding to the
closest point on the left and the closest on the right.  Using the
above kernel with some choice of $\gamma$ let
$$
\rho(x,y) = e^{\gamma |x-y|} - e^{-\gamma |x-y|}.
$$
Let $\xi(x_i) = \tilde{f}_D(x_i) - f_D(x_i)$ represent the noise process.
We find that the conditional distribution of $\xi(x_i)$ to be Normal with mean and variance given by:
$$
\mathbb{E}\xi(x_i) =
\begin{cases}
K(x_i,x_{(1)})\xi(x_{(1)}) & x_i < x_{(1)} \\
K(x_i,x_{(i-1)})\xi(x_{(i-1)}) & x_i > x_{(i-1)} \\
\frac{\rho(x_{(j+1)}, x_i)}{\rho(x_{(j)}, x_{(j+1)})}\xi(x_{(j)}) +  \frac{\rho(x_{(j)}, x_i)}{\rho(x_{(j)}, x_{(j+1)})}\xi(x_{(j+1)}) & x_{(j)} < x_i < x_{(j+1)},
\end{cases}
$$
and
$$
\text{Var}[\tilde{f}_D(x_i)] =
\begin{cases}
1-K(x,x_{(1)})^2 & x_i < x_{(1)} \\
1-K(x,x_{(i-1)})^2 & x_i > x_{(i-1)} \\
1-K(x,x_{(j)})\frac{\rho(x_{(j+1)}, x_i)}
{\rho(x_{(j)}, x_{(j+1)})} - K(x,x_{(j+1)})\frac{\rho(x_{(j)}, x_i)}{\rho(x_{(j)}, x_{(j+1)})} & x_{(j)} < x_i < x_{(j+1)},
\end{cases}
$$
where $x_{(1)} < x_{(2)} < \cdots < x_{(i-1)}$ are
the points $x_1,\ldots,x_{i-1}$ after being sorted into increasing
order.  In using the above algorithm it is only necessary for the data
owner to store the values $x_i$ and $\tilde{f}_D(x_i)$.  When using
the proper data structures e.g., a sorted doubly linked list for the
$x_i$ it is possible to determine the mean and variance using the
above technique in time proportional to $\log(i)$ which is a
significant improvement over the general linear time scheme above
(note that the linked list is suggested since then it is possible to
update the list in constant time).

\section{Conclusion}
\label{section::conclusion}

We have shown how to add random noise to
a function in such a way that
differential privacy is preserved.
It would be interesting to study this method
in the many applications of functional data analysis
\cite{ramsay}.

On a more theoretical note,
we have not addressed the issue of lower bounds.
Specifically, we can ask:
Given that we want to release a differentially private
function, what is the least amount of noise that must necessarily
be added in order to preserve differential privacy?
This question has been addressed in detail
for real-valued, count-valued and vector-valued data.
However, those techniques apply to the case of $\beta = 0$ whereupon
the family $\{P_D\}$ are all mutually absolutely continuous.  In the
case of $\beta > 0 $ which we consider this no longer applies and so
the determination of lower bounds is complicated (for example, since
quantities such as the KL divergence are no longer bounded).

\section{Appendix}

{\bf Proof of Proposition \ref{prop_dp_rkhs}.}
Note that invertibility of the matrix is safely assumed due to Mercer's theorem.
Denote the matrix by $M^{-1}$.  Denote by $P$ the operator $\mathcal{H} \to \mathcal{H}$ defined by
$$
P = \sum_{i=1}^nK_{x_i}\sum_{j=1}^n(M^{-1})_{i,j}\innerh{K_{x_j}}{\cdot}
$$
We find this operator to be idempotent in the sense that $P=P^2$:
\begin{align*}
P^2 &= \sum_{i=1}^nK_{x_i}\sum_{j=1}^n(M^{-1})_{i,j}\innerh{K_{x_j}}{\sum_{k=1}^nK_{x_k}\sum_{l=1}^n(M^{-1})_{k,l}\innerh{K_{x_l}}{\cdot}} \\
 &= \sum_{i=1}^nK_{x_i}\sum_{j=1}^n(M^{-1})_{i,j}\sum_{k=1}^n\innerh{K_{x_j}}{K_{x_k}}\sum_{l=1}^n(M^{-1})_{k,l}\innerh{K_{x_l}}{\cdot} \\
 &= \sum_{i=1}^nK_{x_i}\sum_{j=1}^n(M^{-1})_{i,j}\sum_{k=1}^nM_{j,k}\sum_{l=1}^n(M^{-1})_{k,l}\innerh{K_{x_l}}{\cdot} \\
 &= \sum_{i=1}^nK_{x_i}\sum_{l=1}^n(M^{-1})_{i,l}\innerh{K_{x_l}}{\cdot} \\
& = P.
\end{align*}
$P$ is also self-adjoint due to the symmetry of $M$, i.e.
\begin{align*}
\innerh{Pf}{g} &= \innerh{\sum_{i=1}^nK_{x_i}\sum_{j=1}^n(M^{-1})_{i,j}\innerh{K_{x_j}}{f}}{g}\\
&= \innerh{\sum_{i=1}^n\innerh{K_{x_i}}{g}\sum_{j=1}^n(M^{-1})_{i,j}K_{x_j}}{f}\\
&= \innerh{\sum_{j=1}^nK_{x_j}\sum_{i=1}^n(M^{-1})_{i,j}\innerh{K_{x_i}}{g}}{f}\\
&= \innerh{\sum_{j=1}^nK_{x_j}\sum_{i=1}^n(M^{-1})_{j,i}\innerh{K_{x_i}}{g}}{f}\\
&= \innerh{Pg}{f}.\\
\end{align*}
Therefore,
\begin{align*}
\|f\|_\mathcal{H}^2 &= \innerh{f}{f} \\
& = \innerh{Pf + (f-Pf)}{Pf + (f-Pf)} \\
& = \innerh{Pf}{Pf} + 2 \innerh{Pf}{f-Pf} + \innerh{f-Pf}{f-Pf} \\
& = \innerh{Pf}{Pf} + 2 \innerh{f}{Pf-P^2f} + \innerh{f-Pf}{f-Pf} \\
& = \innerh{Pf}{Pf} + \innerh{f-Pf}{f-Pf} \\
& \geq \innerh{Pf}{Pf} \\
& = \innerh{f}{Pf}.
\end{align*}
The latter term is nothing more than the left hand side in the
statement.  In summary the quantity in the statement of the theorem is
just the square RKHS norm in the restriction of $\mathcal{H}$ to the
subspace spanned by the functions $K_{x_i}$. $\Box$

\bibliographystyle{plain}
\bibliography{dp}

\end{document}